\def\year{2017}\relax
\documentclass[letterpaper]{article} 
\usepackage{aaai17}  
\usepackage{times}  
\usepackage{helvet}  
\usepackage{courier}  
\usepackage{url}  
\usepackage{graphicx}  
\usepackage{multicol}  
\usepackage{definitions}
\usepackage{tikz}
\usetikzlibrary{shapes.geometric}
\usetikzlibrary{arrows.meta,arrows}

\usepackage{amsmath}

\begin{document}
%
\title{Anytime Exact Belief Propagation}
\author{Gabriel Azevedo Ferreira, Quentin Bertrand, Charles Maussion, Rodrigo de Salvo Braz\\
Artificial Intelligence Center\\
SRI International\\
Menlo Park CA, USA \\
}
\maketitle
\begin{abstract}

Statistical Relational Models and, more recently, Probabilistic Programming,
have been making strides towards an integration of logic and probabilistic reasoning.
A natural expectation for this project is that a probabilistic logic reasoning algorithm reduces to a logic reasoning algorithm when
provided a model that only involves 0-1 probabilities,
exhibiting all the advantages of logic reasoning such
as short-circuiting, intelligibility, and the ability to provide proof trees for a query answer.
In fact, we can take this further and require that these
characteristics be present even for probabilistic models
with probabilities \emph{near} 0 and 1, with graceful degradation
as the model becomes more uncertain.
We also seek inference that has amortized constant time complexity
on a model's size (even if still exponential in the induced width of a more directly relevant portion of it) so that it can be applied to huge knowledge bases
of which only a relatively small portion is relevant to typical queries.
We believe that, among the probabilistic reasoning algorithms, Belief Propagation is the most similar to logic reasoning:
messages are propagated among neighboring variables,
and the paths of message-passing 
are similar to proof trees. However, Belief Propagation
is either only applicable to tree models, or approximate
(and without guarantees) for precision and convergence.
In this paper we present work in progress on an Anytime Exact Belief Propagation algorithm
that is very similar to Belief Propagation but is exact
even for graphical models with cycles,
while exhibiting soft short-circuiting, amortized constant time complexity in the model size, and which can provide probabilistic proof trees.

\end{abstract}

\section{Introduction}

Statistical Relational Models \cite{getoor07introduction} and, more recently, Probabilistic Programming,
have been making strides towards \emph{probabilistic logic inference algorithms} that integrate logic and probabilistic reasoning.
These algorithms perform inference
on \emph{probabilistic logical models},
which generalize regular logical models by containing formulas that have a \emph{probability} of being true, rather than \emph{always} being true.

While Statistical Relational Models and Probabilistic Programming focus on higher-level representations such as relations and data structures,
even regular graphical models such as Bayesian networks can
be thought of as probabilistic logic models,
since conditional probability distributions (or factors,
for undirected models) can be described by formulas
and therefore thought of as probabilistic formulas.
This paper focuses on regular graphical models
that will later serve as a basis for higher-level probabilistic logic
models (as further discussed in the conclusion).

Naturally, probabilistic inference algorithms must be able to perform inference
on purely logic models, since these can be seen as 
probabilistic logic models whose formulas have probability 1
of being true.
In this case, it is desirable that the probabilistic inference algorithms  reduce to logic reasoning
in a way that exploits the logic structure of the model nearly as efficiently
as pure logic reasoning algorithms would.
In fact, we should expect even more: if a model (or part of it)
is \emph{near-certain} (with probabilities close to 0 and 1),
then the model is very close to a purely logical model
and it is reasonable to expect a probabilistic inference algorithm to
exploit the logical structure to some degree,
with graceful degradation as the model becomes more uncertain.

\emph{Short-circuiting} is a type of structure that is an important source of efficiency for logic reasoning.
A formula is short-circuited if its value can be determined
from the value of only some of its sub-formulas.
For example, if a model
contains the formula $A \Leftarrow B \vee C \vee D$ and $B$ happens to be true, then a logic reasoning algorithm can conclude that $A$ is true without having to decide whether $C$ and $D$ are true.
However, if a probabilistic reasoning algorithm knows $P(A | B \vee C \vee D) = 1$ and that $P(B) = 0.9$,
it will typically still need to compute $P(C)$ and $P(D)$ in order to compute $P(A)$,
even though it is already possible to affirm that $P(A) \geq 0.9$
without any reasoning about $C$ and $D$.
Providing such a bound can be considered a \define{soft short-circuiting} that
approximates logical short-circuiting as probabilities get closer to 0 and 1, but such ability is absent from most probabilistic inference
methods.

Soft short-circuiting serves as a basis for an anytime, incremental
algorithm that trades bound accuracy for time. Given more time,
the algorithm may determine that $P(C) \geq 0.8$ independently of $B$,
perhaps by recursively processing another rule $P(C | E \vee F)$,
which allows it to increase its lower bound $P(A) \geq 0.98$,
a tight bound obtained without ever reasoning about some
potentially large parts of the model (in this case, rules involving $D$
or adding information about $C$).

There are several probabilistic inference algorithms, discussed in
Section \ref{sec:related-work}, that produce bounds on query probabilities.
However, we find that most of these algorithms do not exhibit another
important property of logic reasoning algorithms:
a \define{time complexity for inference that is amortized constant in the size of the entire model} (although still exponential
on the induced width of the portion of the model
that is \emph{relevant} for computing the current bound).
This is achieved by storing formulas in a model in hash tables indexed by the variables they contain, and looking up formulas only as needed
during inference as their variables come into play.
An important application for 
probabilistic logic reasoning in the future is reasoning
about thousands or even millions of probabilistic rules
(for example in knowledge bases learned from the Web such as NELL \cite{mitchell15never}),
for which this property will be essential.

Finally, we are also interested in a third important
property of logic reasoning algorithms:
the ability to \define{produce an intelligible trace of its inference}
(such as a proof or refutation tree)
that serves as a basis for explanations
and debugging.
In probabilistic reasoning, Belief Propagation
is perhaps the closest we get to this, since local message-passing
is easily understood and the tree of messages can be used as 
a proof tree. However, BP only returns correct marginal probabilities
for graphical models without cycles, 
and non-guaranteed approximation and convergence for general graphical
models. 

In this paper, we present work in progress on \emph{Anytime Exact Belief Propagation},
an algorithm that exhibits the three properties described above:
it incrementally computes bounds on marginal probabilities that
can be provided at any time, whose accuracy can be traded off for time,
and that eventually converge to the exact marginal;
it has time complexity amortized constant in the size of the entire model;
and it produces a tree of local messages that can be used to 
explain its conclusion.

\section{Related Work}
\label{sec:related-work}

The most obvious candidates for probabilistic logic reasoning approaches
that exhibit logic properties with graceful degradation are the ones
based on logic programming: 
Bayesian Logic Programs \cite{kersting00},
PRISM \cite{sato&kameya97}, 
Stochastic Logic Programs \cite{muggleton95stochastic},
and ProbLog \cite{deraedt04probabilistic}.
While these approaches can be used with sampling, they typically
derive (by regular logic programming methods)
a proof tree for the query and evidence in order to determine which portion of the model is qualitatively relevant to them.
Only after that does the probabilistic reasoning starts.
This prevents selecting portions of the model based on \emph{quantitative} relevance.

More recently, the ProbLog group has proposed two methods for anytime inference:
\cite{RenkensAAAI14}
successively finds multiple explanations for a query,
each of them narrowing bounds on the latter's probability.
However, finding an explanation requires inference on the entire model,
or at least on a qualitatively relevant portion that may include
\emph{quantitatively} irrelevant knowledge.
\cite{VlasselaerIJCAI15}
proposes a method based on forward reasoning with iterative deepening.
Because the reasoning goes forward, there is no clear way to limit the inference
to the portion most relevant to a query (which is a form of \emph{backward} inference),
and no selection for more likely proofs.

Our calculation of bounds is equivalent to the one presented in 
\cite{leisink03bound},
but that work does not attempt to focus on relevant portions of a model
and does not exploit the graphical model's factorization as much as our method and Variable Elimination do.
Box propagation 
\cite{mooij08bounds} 
propagates \emph{boxes}, which are looser bounds
than ours and Leisink \& Kappen's bounds.
In fact, their method can easily use these tighter bounds, but in any case
it does not deal with cycles, stopping the unrolling of the model (the \emph{Bethe tree}) once
a cycle is found. Ihler \cite{ihler07b}
presents a similar method that does not stop
when a cycle is found, but is not guaranteed to converge to the exact marginal.

Liu et al \cite{aaai17a}
present a method very similar to ours based on growing an AND-OR search tree
from the query and bounding the not-yet-processed remaining of the model.
As more of the tree is expanded, the better the bounds become.
The main difference from our method is that the AND-OR search tree has a child
per value of each random variable, making it arguably less intelligible
and harder to use as a proof tree and to generalize to richer representations such as
Statistical Relational Models.

\section{Background}

\subsection{Graphical Models}

\define{Graphical models} are a standard framework for reasoning with uncertainty.
The most common types are Bayesian networks
and Markov networks.
In both cases, a \define{joint probability distribution}
for each assignment tuple \bx to $N$ random variables
is defined as a normalized product of non-negative real functions $\{\phi_i\}_{i\in 1..K}$,
where $1..K$ is short for $\{1,\dots,K\}$,
each of them applied to a subtuple $\bX_i$ of \bX:\footnote{For simplicity, we use the same symbols for both random variables and their values, but the meaning should be clear.}
\begin{align*}
P(\bX) = \frac{1}{Z}\prod_{i=1} ^K \phi_i(\bX_i),
\end{align*}
where $Z$ is a normalization constant equal to $\sum_{\bX} \prod_i \phi_i(\bX_i)$.
Functions $\phi_i$ are called \define{factors}
and map each assignment on their arguments to a \define{potential},
a non-negative real number that represents
how likely the assignment $\bX_i$ is.
This representation is called \define{factorized} due to its
breaking the joint probability into this product.
In Bayesian networks, $K = N$ and factors are conditional probabilities
$P(X_i | Pa_i)$, for each random variable $X_i$ in \bX, where $Pa_i$
are its parents in a directed acyclic graph.

For succinctness, we often do not explicitly write the arguments to factors:
\begin{align*}
P(\bX) = \frac{1}{Z}\prod_i \phi_i(\bX_i) = \frac{1}{Z}\prod_i \phi_i.
\end{align*}

The \define{marginal probability (MAR)} problem consists of computing
\begin{align*}
P(\bQ) = \sum_{\bX \setminus \bQ} P(\bX),
\end{align*}
where $\bQ$ is a subtuple of \bX
containing queried variables, and $\sum_{\bX \setminus \bQ}$ is the summation over all variables in \bX but not in \bQ.
It can be shown that $P(\bQ)=\frac{1}{Z_{\bQ}} \sum_{\bX \setminus \bQ} \prod_i \phi_i$
for $Z_\bQ$ a normalization constant over $\bQ$.
Therefore, because $Z_{\bQ}$ is easily computable if $|\bQ|$ is small,
the problem can be simply reduced to computing a \define{summation over products of factors}, which the rest of the paper focuses on.

We denote the \define{variables} (or neighbors) of a factor $\phi$ or set of factors $M$
as $Var(\phi)$ and $Var(M)$.
The \define{neighbors} $neighbors_M(V)$ of a variable $V$ given a set of factors $M$ is defined as the set of factors $\{ \phi \in M : V \in Var(\phi) \}$.
We call a set of factors a \define{model}.
The \define{factor graph} of a model $M$ is the graph with variables
and factors of $M$ as nodes and with an edge between
each factor and each of its variables.

\subsection{Belief Propagation}

Belief Propagation \cite{Yedidia:2003:UBP:779343.779352}
is an algorithm that computes the marginal probability
of a random variable given a graphical model whose factor graph
has no cycles.

Let $M$ be a set of factors and $P_M$ be 
the probability distribution it defines.
Then, for a set of variables $Q \subseteq Var(M)$,
we define:
\begin{align*}
& P_M(Q) \propto m^M_{. \leftarrow Q} \text{ (note that $m^M_{\phi \leftarrow Q}$ does not depend on $\phi$).} \\
& m^M_{V \leftarrow \phi} = \sum_{S} \phi \prod_{S^j \in S} m^{M^j}_{\phi \leftarrow S^j} \\
& \text{\qquad where } \{S^1,\dots,S^n\} \defeq Var(\phi) \setminus V,\\
& \text{\qquad       } M^j \text{ is the set of factors in }M \setminus \{ \phi \}\text{ connected to }S^j,\\
& m^M_{\phi \leftarrow V} = \prod_{\phi^j \in neighbors_M(V)} m^{M^j}_{V \leftarrow \phi^j}, \\
& \text{\qquad where } \{\phi^1,\dots,\phi^n\} \defeq neighbors_M(V),\\
& \text{\qquad       } M^j \text{ is the set of factors in }M\text{ connected to }\phi^j.\\
\end{align*}	
Note that each message depends on a number of \define{sub-messages}.
Since the factor graph is a tree (it has no cycles),
each sub-message involves a disjoint set of factors $M^j$.
This is crucial for the correctness of BP because it allows the computation to be separately performed for each branch.

If a graphical model has cycles, an iterative version of BP, \define{loopy BP}, can still be applied to it
\cite{Yedidia:2003:UBP:779343.779352}.
In this case, since a message will eventually depend on itself, we use its value from a previous iteration,
with random or uniform messages in the initial iteration.
By iterating until a convergence criterion is reached, loopy BP provides distributions, called \define{beliefs}, that in practice often approximate the marginal of the query well.
However, loopy BP is not guaranteed to provide a good approximation, or even to converge.

\subsection{Anytime Belief Propagation}

\begin{figure}
	\vspace{-0.3cm}
	\centerline{\includegraphics[width=\columnwidth]{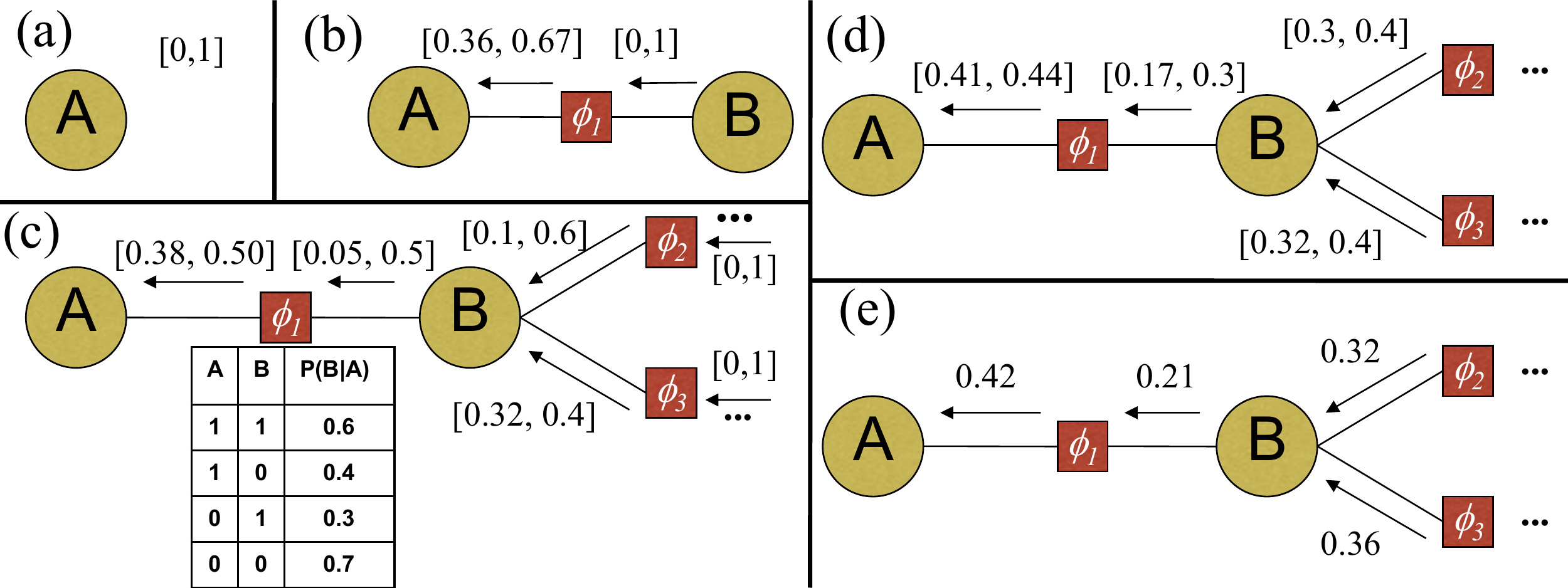}}
	\vspace{-0.3cm}
	\caption{\small Bound propagation on a binary variables network.}
	\label{fig:bound-propagation}
\end{figure}

Even though BP is based on local computations between neighboring nodes,
it only provides any information on the query's answer once it has analyzed the entire model,
even if some parts of the model have a relatively small influence on the answer.
This goes against our initial goal  providing information on the query's answer even after analyzing only a (hopefully more relevant) portion of the model.

\begin{figure*}[tp]
	\begin{multicols}{2}
		\includegraphics[trim = 0cm 3cm 0cm 0cm,
		width=\textwidth]{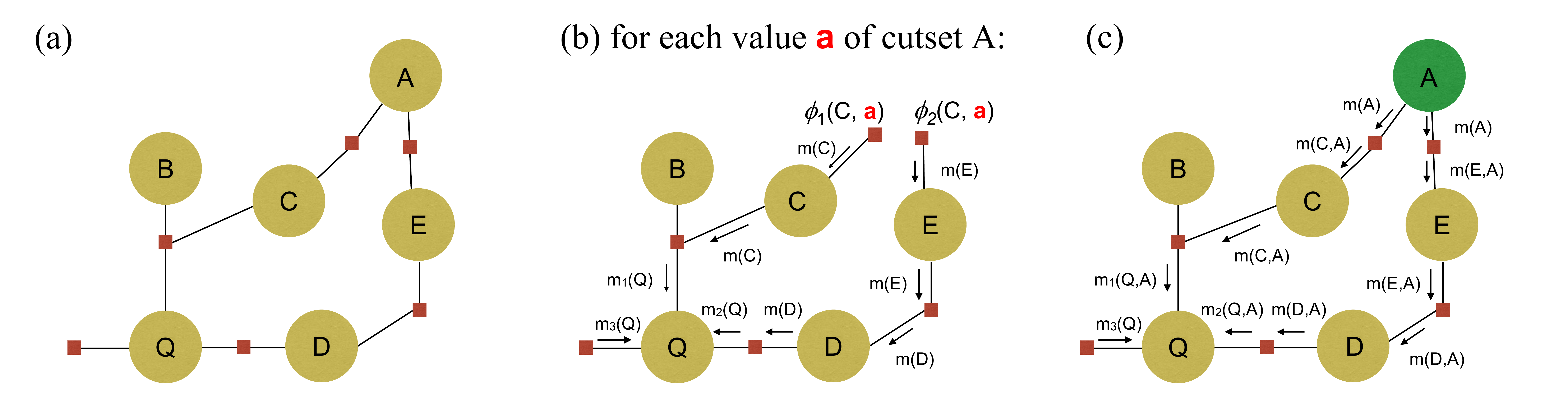}
	\end{multicols}
	\caption{\small Cutset conditioning. BP does not apply to model with cycles (a). We can fix a cycle cutset $\{A\}$ to a value $a$ and replace
		its factors by ones that take $a$ into account and no longer have
		a parameter $A$, eliminating the cycle, computing $P(Q,a)$ for each $a$,
		and obtaining $P(Q)$ as $\sum_a P(Q,a)$ (b).
		Alternatively but equivalently, we can fix $A$ and have a single pass of BP propagate messages that are
		functions of $A$ (c).}
	\label{fig:cutset-conditioning}
\end{figure*}

\define{Anytime BP} \cite{desalvobraz09anytime}
is an algorithm based on (loopy) BP that
computes iteratively improved bounds on a message.
A \define{bound} (following definitions in \cite{mooij08bounds}) on a message $m$ is any set of messages
to which $m$ is known to belong.
Anytime BP (and, later, Anytime Exact BP)
only use bounds that are convex sets of messages,
and that can therefore be represented
by a finite number of messages (the bounds \define{extrema}),
the convex hull of which is the entire bound.
Initially, the bound on a message on a variable $V$
is the \define{simplex} $\mP(V)$,
the set of all possible probability distributions on $V$,
and whose extremes are the distributions that place
the entire probability on a single value.
For example, if $V$ is a boolean random variable,
its simplex
is the set $\{ \ifte{V = true}{1}{0}, \ifte{V = false}{1}{0} \}$.

It turns out that the computation of a message $m$ given its sub-messages
is a convex function. Therefore, given the bounds on sub-messages
represented by their extremes, we can compute a bound $b(m)$ on $m$
by computing the extremes of this bound,
each extreme being equal to the message computed from
a combination of extremes to the sub-messages.
This provides a finite set of extremum messages that define $b(m)$
and can be used to compute further bounds.

Figure \ref{fig:bound-propagation} shows an example of Anytime Belief propagation on a factor network.
The algorithm provides increasingly improving bounds on the belief $m(A)$
on query $A$,
by first returning the simplex $\mP(A)$ as a bound,
then returning the bound computed from simplex sub-messages,
and then successively refining this bound by selecting one or more of the sub-messages,
obtaining tighter bounds on these sub-messages, and recomputing
a tighter bound for $m(A)$.
At every step from (b) to (d), factors are included in the set so as to complete some random variable's blanket
(we do not show the expansions from (d) to (e), however, only their consequences).
We include the table for factor 
$\phi_1$ but omit the others.
For simplicity, the figure uses binary variables only and shows
bounds as the interval of possible probabilities for value $1$, but it applies to multi-valued variables as well.

This incrementally processes the model from the query,
eventually processing it all and producing an exact bound
on the final result.
Like BP, Anytime BP is exact only for tree graphical models,
and approximate for graphical models with cycles
(in this case, the bounds are exact for the \emph{belief},
that is, they bound the approximation to the marginal).
The main contribution of this paper is Anytime \emph{Exact} BP,
which is a bounded
versions of BP that is exact for any graphical models,
presented in Section \ref{sec:anytime-exact-belief-propagation}.

\subsection{Cycle Cutset Conditioning}

If a graphical model has cycles, an iterative version of BP, \define{loopy BP}, can still be applied to it
\cite{Yedidia:2003:UBP:779343.779352}.
However, loopy BP is not guaranteed to provide a good approximation, or even to converge.

\define{Cycle cutset conditioning} \cite{pearl88probabilistic} is a way of using BP to solve a graphical model $M$ with cycles.
The method uses the concept of \emph{absorption}:
if a factor $\phi_i(\bX_i)$ has some of its variables $\bV \subseteq \bX_i$ set
to an assignment $\bv$, it can be replaced by a new factor $\phi'$,
defined on $\bX \setminus \bV$ and $\phi'(\bX_i \setminus \bV) = \phi_i(\bX_i \setminus \bV, \bv)$.
Then, cutset conditioning consists of selecting $C$, a \define{cycle cutset} random variables
in $M$ such that, when fixed to a value $c$, gives rise through absorption in all factors involving variables in $C$ to a new graphical model $M_c$ \emph{without cycles} and defined on the other variables $\bX \setminus C$
such that $P_{M_c}(\bX \setminus C) = P_M(\bX \setminus C, c)$.
The marginal $P_M(Q)$ can then be computed by
going over all assignments to $C$ and solving the corresponding
$M_c$ with BP:
\begin{align*}
P_M(Q) & = \sum_{\bX \setminus Q} P_M(\bX) \\
& =  \sum_c \sum_{\bX \setminus (Q \cup C)} P_M(\bX \setminus C, c) \\
& =  \sum_c \sum_{\bX \setminus (Q \cup C)} P_{M_c}(\bX \setminus C)\\
& =  \sum_c P_{M_c}(Q) \text{ (using BP on $M_c$)}.
\end{align*}

Figure \ref{fig:cutset-conditioning} (a) shows a model with
a cycle. Panel (b) shows how cutset conditioning for cutset $\{A\}$
can be used to compute $P(Q)$: we successively fix $A$ to each
value $a$ in its domain, and use absorption to
create two new factors $\phi_1' = \phi_1(C,a)$ and $\phi'_2(E) = \phi_2(E,a)$. This new model does not contain any cycles and BP computes $P(Q,a)$. The overall $P(Q)$ is then computed as $\sum_a P(Q,a)$.
Now, consider that the messages computed across the reduced model
that depend on $a$ can be thought of as \emph{functions} of $a$.
From that angle, the multiple applications of BP for each $a$
can be thought of as a single application of BP in which
$A$ is a fixed, free variable that is not eliminated and
becomes a parameter in the propagated messages (panel (c)).
This has the advantage of computing all messages that do \emph{not}
depend on $a$ only once.

While cutset conditioning solves graphical models with cycles exactly,
it has some disadvantages.
Like standard BP, cutset conditioning requires the entire model to be processed before providing useful information.
In fact, simply finding a cutset already requires going over the entire model, before inference proper starts.
Besides, its cost grows exponentially in the size of the cutset,
which may be larger than the induced tree width.
Our main proposal in this paper, Anytime Exact Belief Propagation,
counters those disadvantages by processing the model
in an incremental way, providing a hard bound around the exact
solution, determining the cutset during inference,
and summing out cutset variables as soon as possible
as opposed to summing them out only at the end.

\section{Anytime Exact Belief Propagation}
\label{sec:anytime-exact-belief-propagation}

We are now ready to present the main contribution of this paper,
Anytime Exact Belief Propagation (AEBP).
Like cutset conditioning, the algorithm applies
to any graphical models, including those with cycles.
Unlike cutset conditioning, it does not require
a cutset to be determined in advance,
and instead determines it \emph{on the fly},
through local message-passing.
It also performs a gradual discovery of the model, providing bounds on the final result as it goes.
This is similar to Anytime BP, but
Anytime Exact BP, as the name implies,
provides bounds on the exact query marginal probability
and eventually converges to it.

We first provide the intuition for Anytime Exact BP through an example.
Consider the graphical model in Figure \ref{fig:anytime-exact-belief-propagation-simple-loop} (the full model is shown in (e)).
If we simply apply Anytime BP to compute $P(Q)$, messages will be computed in an infinite loop.
This occurs because Anytime BP has no way of identifying loops.
AEBP, on the other hand, takes an extra measure in this regard: when it requests a new bound from one of the branches
leading to a node, it also provides the set of factors known so far to belong to the \emph{other} branches.
Any variable that is in the branch and is connected to these external factors must necessarily be a cutset variable.
Upon finding a cutset variable $C$, AEBP considers it fixed and does not sum it out,
and resulting messages are \emph{functions} of $C$ (as well we the regular variable for which we have a message).
Branches sharing a cutset variable $C$ will therefore return bounds that are functions of $C$.
Cutset variables are summed out only after messages from all branches containing $C$ are collected.

While the above procedure is correct, delaying the sum over cutset variables until the very end
is exponentially expensive in the number of them.
Figure \ref{fig:anytime-exact-belief-propagation-diamond}
shows an example in which a cutset variable ($G$) that occurs only in an inner cycle
is summed out when that cycle is processed (at node $E$), while the more global cutset variables ($C$ and $F$)
are summed out at the end, when the more global cycle is processed (at node $A$).

\begin{figure*}
	\begin{multicols}{2}
	\includegraphics[trim = 1cm 2cm 1cm 2cm, width=6.8in]{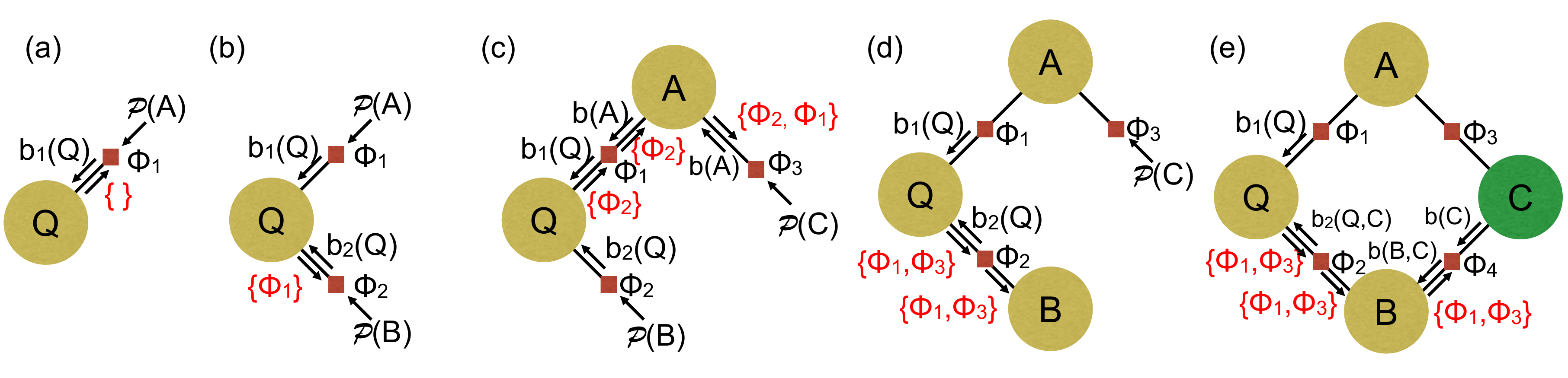}
	\label{fig:anytime-exact-belief-propagation-simple-loop}
	\end{multicols}
	\caption{\small Discovering a cutset in AEBP. $Q$ requests
	a bound on $Q$ from $\phi_1$ while telling it that
	no factors have been observed (in \textcolor{red}{red}) in other components (empty set \textcolor{red}{$\{\}$}) (a).
	The bound is computing by only assuming that the message from $A$
	is the simplex $\mP(A)$.
	It then requests a bound from $\phi_2$,
	this time telling it about $\phi_1$ having
	already been observed in a different component (b).
	In (c), it requests a better bound from $\phi_1$,
	this time telling it about having $\phi_2$,
	which triggers a request from $\phi_3$ to $A$
	with information of having previously observed $\{\phi_1,\phi_2\}$.
	In (d), a request goes all the way to $B$ with information having
	observed, among others, $\phi_3$. When $B$ requests a bound from $\phi_4$ (e),
	it is detected that $\phi_4$'s argument $C$ is also an argument
	of previously observed factor $\phi_3$, which leads to making
	$C$ a cutset variable. Bounds are then computed as a \emph{function}
	of $C$, as opposed of summing it out, all the way back to $Q$,
	where $C$ will eventually be summed out.
}
\end{figure*}

Algorithm 1 presents the general formulation.
It works by keeping track of each \define{component}, that is,
a branch of the factor graph rooted in either a variable or factor,
its $Node$,
and computing a message from $Node$ to some other requesting node
immediately outside the component.
Each component is initially set with the external factors
that have already been selected by other components.
The message is on a variable $V$ (this is $Node$ itself if $Node$
is a variable, and some argument of $Node$ if it is a factor).
In the first update, the component sets the bound to the simplex on $V$
and creates its children components: if $Node$ is a variable,
the children components are based on the factors on it that are not already external;
if it is a factor, they are based on its argument variables.

From the second update on, the component selects a child,
updates the child's external factors by including that child's
siblings external factors,
updates the child's bound,
updates its own set of factors by including the child's newly discovered
factors, and computes a new bound.
If $Node$ is a factor, this is just the product of the bounds of its children.
If $Node$ is a variable, this is obtaining by
multiplying $\phi$ and children bounds, and summing out
the set of variables 
$S$.
$S$ is the set of variables that occur only inside $\mC$,
(which excludes $V$ and cutset variables connected to external factors,
but \emph{does} include cutset variables that occur only inside this component.
This allows cutset variables to be eliminated as soon as possible in the process.
To compute the marginal probability for a query $Q$,
all that is needed is to create a component for $Q$ without
external factors and update it successively until it converges to
an exact probability distribution. 
During the entire process, even before convergence, this component tree can be used as a trace of the inference process, indicating how each message has been computed from sub-messages so far, similarly to a probabilistic proof or refutation tree in logic reasoning.

\begin{codebox}
	\Procname{\proc{\textbf{Update($\mathcal{C}$)}}}
	\zi \mC is a \define{component}, defined as a tuple \\
	$(V,Node, Bound, M, ExteriorFactors, Children)$\\
	where:
	\zi $V$: the variable for which a message is being computed
	\zi $Node$: a variable or factor \emph{from} which the \\
	\qquad \quad message on variable $V$ is being computed
	\zi $Bound$: a bound on the computed message
	\zi $Factors$: the set of factors selected for this message already
	\zi $ExternalFactors$: set of factors already observed \emph{outside}\\
	the component, and used to identify new cutset variables.
	\zi $Children$: components for the sub-messages of this message.
	\li \If first update \Then
	\li		$Bound \gets \mP(V)$
	\li \If $Node$ is variable \Then
	\li     $Factors \gets$\text{ factors with $Node$ as argument} \\
		\hspace{3cm} and \emph{not} in $ExternalFactors$	
	\li     $Children \gets $ components based on each \\
	\hspace{3.2cm}factor in $Factors$ \\
	\hspace{3.2cm}and $ExternalFactors$ set \\
	\hspace{3.2cm}to $\mC.ExternalFactors$
	\li     \Else \text{ // $Node$ is factor $\phi$}
	\li     $Bound \gets \sum_{C \cup S} \phi \prod_{Ch \in Children}$
	\li     $Children \gets $ components based on each variable\\
\hspace{3.2cm}argument (that is, neighbor) of $\phi$\\
\hspace{3.2cm}and $ExternalFactors$ set \\
\hspace{3.2cm}to $\mC.ExternalFactors \setminus \{\phi\}$
	\End
	\li \Else
	\li $Child \leftarrow chooseNonConvergedChild(Children)$
	\li $Child.ExternalFactors \gets $\\
	\qquad \qquad $ExternalFactors \quad \cup $\\
	\qquad \qquad \qquad $\bigcup \limits_{Ch \in Children \setminus \{Child\}} Ch.Factors$
	\li \textbf{\textsc{Update}}$(Child)$
	\li $Factors \gets Factors \cup Child.Factors$
	\li $ChildrenBoundProduct \gets \prod_{Ch \in Children} Ch.Bound$
	\li \If $Node$ is variable \Then
	\li     $Bound \gets ChildrenBoundsProduct$
	\li     \Else \text{ // $Node$ is factor $\phi$}
	\li     $S \gets $ variables in $ChildrenBoundProduct$ \\
	\hspace{3cm}not in any factor in $ExternalFactors$
	\li     $Bound \gets \sum_{S} \phi \prod_{Ch \in Children} Ch.Bound$
	\End
	\End
\end{codebox}
{\centering \textbf{Algorithm 1}: Anytime Exact Belief Propagation.}
\label{alg:anytime-exact-belief-propagation}

\begin{figure}
	\includegraphics[trim = 0cm 1.2cm 0cm 1cm,
	width=\columnwidth]{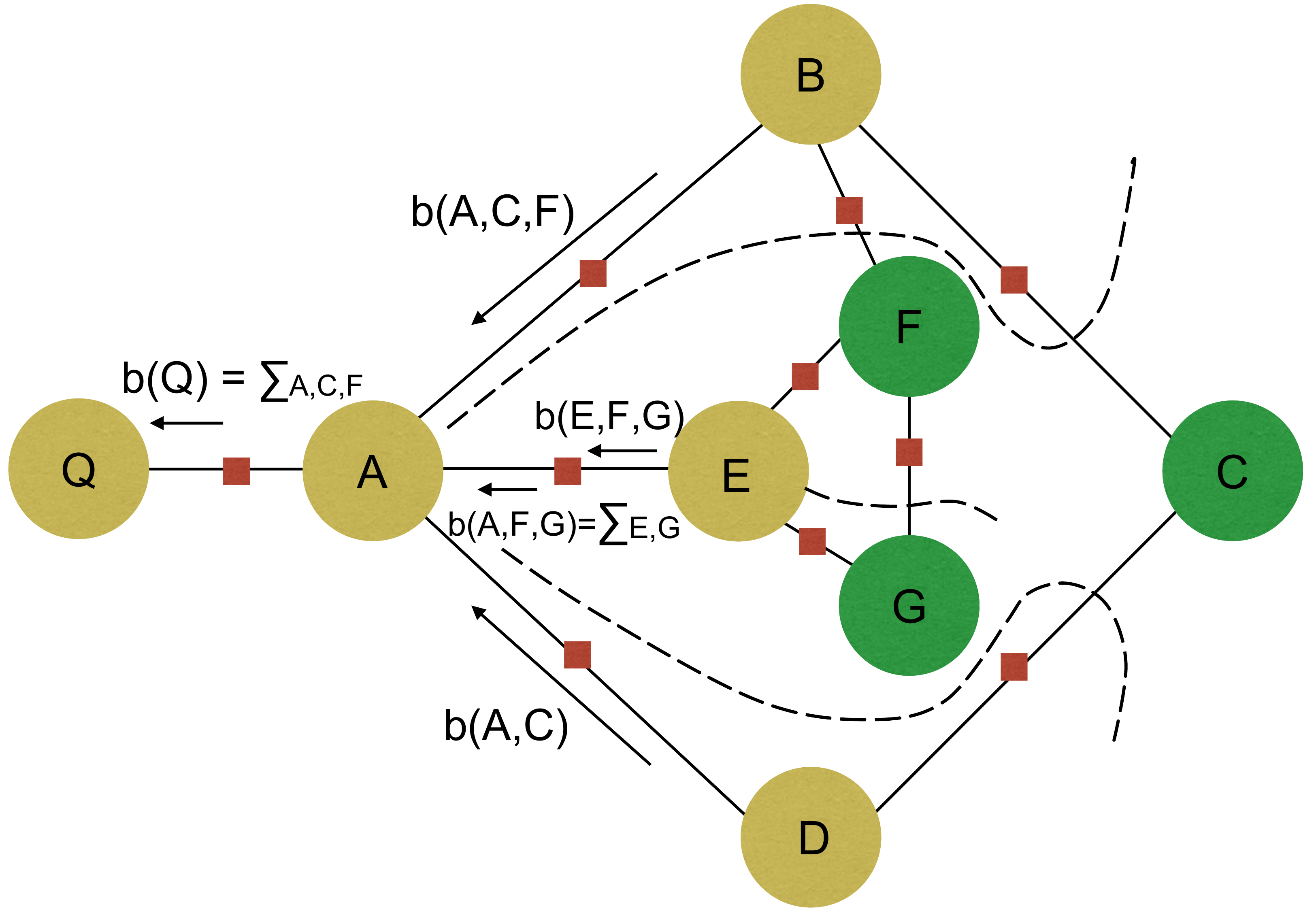}
	\caption{\small An example of AEBP that eliminates portions of the cutset separately. The dashed lines show how branches are followed from both $A$ and $E$. $G$ is detected as a cutset variable with respect to $E$
		because it connects two branches of $E$, so it is summed out along
		with $E$ when producing the bound towards $A$. $F$ is also detecting
		while exploring the branches of $E$, but it connects to a branch of $A$
		so it is not summed out when $E$ is finished, being instead summed out
		when $A$ is finished. $C$ connects two branches of $A$ so it is also
		summed out along with $A$.}
	\label{fig:anytime-exact-belief-propagation-diamond}
\end{figure}

\tikzset{neuron/.style={shape=circle, minimum size=.9cm, inner sep=0, draw, font=\small},
            			 io/.style={neuron, fill=gray!20}}

\section{Conclusion}

We presented our preliminary work on Anytime Exact Belief Propagation,
an anytime, exact inference method for graphical models that provides
hard bounds based on a neighborhood of a query.
The algorithm aims at generalizing the advantages of logic reasoning
to probabilistic models, even for dependencies that are not certain, but near certain.

Future work includes finishing the implementation, evaluating it on
benchmarks, and generalizing it higher-level logic representations
such as relational models and probabilistic programs. To achieve that,
we will employ techniques from the lifted inference literature
\cite{poole03first,desalvobraz07thesis,milch08lifted,vanDenBroeck11lifted,kersting12lifted}
as well as probabilistic inference modulo theories
\cite{desalvobraz16probabilistic}.

\bibliographystyle{aaai}
\bibliography{bib}

\comment {
\appendix

\section{Belief Propagation with Separator Conditioning: an exact message-passing inference algorithm for any graphical model}

Belief Propagation is equivalent to the inference calculation  when the graphical model does not present any cycles. 
The same can not be stated for general graphical models, where loopy belief propagation algorithm may not converge. If it does,  not necessarily it converges to an optimal solution.

We based our anytime algorithm, therefore, upon an adaptation of the belief propagation, where we ``break'' the cycles by summing  out some variables in a later step of the process. This provides us with an exact elimination algorithm.

    \subsection{Definitions}
We provide some definitions  that  are useful in defining the algorithm.
    \subsubsection*{Partition}
    Let M be a set of factors. A partition $Pt_M$ of M id a tuple of nonempty sets of factors
    $$Pt_M = (Pt^1,\dots,Pt^k)$$
    such that each element of M is exactly in one of the subsets $Pt^i$. That is to say that :
    
    \begin{itemize}
    \item for $i = 1,\dots,k \ Pt^i \neq \emptyset$
    \item for $i,j = 1,\dots,k \ i\neq j \ , \  Pt^i\cap Pt^j = \emptyset$
    \item $\bigcup_{i=1,\dots,k}Pt^i=M $
    \end{itemize}
    \subsubsection*{Separator}
    
	A separator of set of a partition $Pt_M$ is defined as following:
\begin{align*}
	Sep(Pt_M) =\{V: & \exists \ i,\ j\ , \ i\neq j, \ s.t.\ \\
    				& V\in Var(Pt^i)\cap Var(Pt^j) \}
\end{align*}  
	
	\subsection{BP with conditioning}
	Let $M$ be a set of factors.
	Then, for a set of variables $Q \subseteq Var(M)$,
	we can employ Belief Propagation with Separator Conditioning (S-BP) to compute $P(Q)$:
	
	\begin{align*}
		& P_M(Q) \propto m^{M}_{. \leftarrow Q} \text{ (note that $m^M_{\phi \leftarrow Q}$ does not depend on $\phi$).} \\
		& m^{M}_{V \leftarrow \phi} = \sum_C \sum_{S\setminus C} \phi \prod_{S^j \in S} m^{Pt^j}_{\phi \leftarrow S^j \cup C} \\
		& \text{\qquad where } \{S_1,\dots,S_n\} = Var(\phi) \setminus V,\\
		& \text{\qquad         $(Pt^1,\dots,Pt^n)$ is any $n$-partition of $M \setminus \{\phi\}$,} \\
		& \text{\qquad         $C$ is the $Sep(Pt\setminus\{\phi\})$,} \\
		& m^{M}_{\phi \leftarrow V} = \sum_C \prod_{\phi^j \in neighbors_M(V)} m^{Pt^j}_{V \cup C \leftarrow \phi^j}, \\
		& \text{\qquad where } neighbors_M(V) \text{ are the factors on $V$ in $M$}, \\
		& \text{\qquad         $(Pt^1,\dots,Pt^n)$ is any $n$-partition of $M$ such that $\phi^j \in Pt^j$,} \\
		& \text{\qquad         $C$ is $Sep(Pt)\setminus V$.}
	\end{align*}	
	
	\begin{proof}
		We prove that $m^M_{\phi \leftarrow V} \propto P_M(V)$ and $m^M_{V \leftarrow \phi} \propto P_M(V)$ by induction on the size of $M$.
		If $|M|=0$, $m^M_{\phi \leftarrow V} = 1$, which when normalized is equal to $Uniform(V) = P_M(V)$.
		$m^M_{V \leftarrow \phi}$ is undefined for $|M|=0$, but for $|M|=1$ (that is, $M=\{\phi\}$),
		$C=\emptyset$ and $m^M_{V \leftarrow \phi} = \sum_S \phi(V,S) \propto P_M(V)$.
		
		If $|M|>0$, 
		\begin{align*}
			m^M_{\phi \leftarrow V} & = \sum_C \prod_{\phi^j \in neighbors_M(V)} m^{Pt^j}_{V \cup C \leftarrow \phi^j}\\
			& \propto \sum_C \prod_{\phi^j \in neighbors_M(V)} P_{Pt^j}(V \cup C)\\
			& = \sum_C \prod_{\phi^j \in neighbors_M(V)} \sum_{Var(Pt^j)\setminus(V \cup C)}Pt^j\\
			& = \sum_C \sum_{Var(M)\setminus(V \cup C)}\prod_{\phi^j \in neighbors_M(V)} Pt^j\\
			& = \sum_{Var(M)\setminus V}\prod_{\phi^j \in neighbors_M(V)} Pt^j\\
			& = \sum_{Var(M)\setminus V} M\\
			& \propto P_M(V)\\
		\end{align*}
        
		\begin{align*}
			m^M_{V\leftarrow \phi} & = \sum_C \sum_{S \setminus C} \phi \prod_{S^j \in S} m^{Pt^j}_{\phi \leftarrow S^j \cup C} \\
			& \propto \sum_C \sum_{S \setminus C} \phi \prod_{S^j \in S} P_{Pt^j}(S^j \cup C)\\
			& = \sum_C \sum_{S \setminus C} \phi P_{M \setminus{\phi}}(S \cup C)\\
			& \propto \sum_C \sum_{S \setminus C} \phi \prod_{\phi' \in M \setminus \phi} \phi'\\
			& \propto \sum_C \sum_{S \setminus C} P_{M}(S \cup C)\\
			& = \sum_{S \cup C} P_{M}(S \cup C)\\
			& = \sum_{Var(M) \setminus V} P_{M}(Var(M) \setminus V)\\
			& = P_M(V).
		\end{align*}
	\end{proof}
	
	S-BP reduces to regular BP if the model does not contain loops. Otherwise,
	it uses separators to ensure incoming messages are
	computed on components (sets of factors) that are disjoint given a separator set.
	The separator set for a message is marginalized out at the end of its computation.
	
	Unlike BP, S-BP is not a deterministic algorithm, since there may be several
	choices for $Pt$.
	It can be made deterministic by representing the set of factors $M$
	as a tree of all partitions to be used in all message computations.
	Later, we will use partial specifications of $M$ in this manner
	to represent bounds on $M$.
	
\section{Bounding Functions}
	
	The following definition introduces the concept of iteratively bounding the value $f(a)$ of a function $f$ for a \emph{unknown} domain element $a$.
	A \emph{bound} on some mathematical object $\alpha$ is simply some set containing $\alpha$.
	A bounding function does two things given a bound on $a$: it computes a bound on $f(a)$,
	and computes a strictly tighter bound (a smaller set) on $a$ in the process.
	This yields a bounding sequence generated by an iterative process in which partial or complete ignorance on $a$ and therefore $f(a)$ leads to tighter and tighter bounds on $a$ and $f(a)$ until we converge to the exact values.
	
	When we apply this framework to  our particular application of exact BP, $a$ will be the model $M$ and the cutset $Q$, and $f$ will be the messages on them. The bounds on $(M,Q)$ will be based keep partial descriptions of $M$ and $Q$ and convergence is obtained when they are fully analyzed. We will have a bound on the messages during the entire process, making the algorithm anytime. However, we define the notions of bounding functions and sequences in an abstract way that applies to any functions and sets, and those in probabilistic inference are just one case.
	
	The intuition behind a bounding function is that, instead of receiving an exact argument $a$,
	it receives a set of possible arguments $A' \subseteq A$ containing $a$ (a bound on $a$), and returns the possible results $B = f(A')$ (and thus guaranteed to contain $f(a)$), and also a strictly smaller $A'' \subset A'$ determined during the processing. $A''$ can then be fed again to $\bar{f}$ to obtain an even better bound on $f(a)$.
	
	\begin{definition}[Bounding function]
		Let $f : A \rightarrow B$ be any function, and $a \in A$.
		A function $\bar{f} : 2^A \rightarrow (2^A, 2^B)$ is a \emph{bounding function of }$f$
		with respect to $a$ if, for any $A' \subseteq A$ for which $a \in A'$,
		$\bar{f}(A') = (A'', B)$ satisfies $a \in A''$, $A'' \subset A'$, and $B = \{ f(a'') : a'' \in A'' \}$.
	\end{definition}
	
	\begin{theorem}[Bounding Sequence]
		\label{the:bounding-sequence}
		Let $A$ be a finite, discrete set,
		$f : A \rightarrow B$ be any function with domain $A$,
		and $\bar{f}$ a bounding function of $f$.
		Then there is $n \in \mathbb{N}$ and a sequence defined by
		\begin{align*}
			& A_0 = A; B_0 = B \\
			& (A_i , B_i) = \bar{f}(A_{i - 1}), i = 1,\dots,n
		\end{align*}
		such that, for all $i \geq 1$,
		$B_i \subset B_{i-1}$, $f(a) \in B_i$, and $B_n = \{ f(a) \}$.
	\end{theorem}
	
	\begin{proof}
		The existence of the sequence and its properties, except for $B_n = \{f(a)\}$, comes directly from the definition
		of a bounding function.
		We prove $B_n = \{f(a)\}$ by induction on $|A_i|$.
		If $|A_i|=1$, $A_i=\{a\}$ and $B_i = \{f(a)\}$
		by the definition of bounding function.
		If $|A_i| > 1$, then $|A_{i+1}| < |A_i|$, and by induction
		it defines a sequence from $(A_{i+1}, B_{i+1})$ to $(A_n, B_n)$ such that $B_n = \{ f(a) \}$.
		Therefore, $A_i$ defines a sequence from $(A_i, B_i), (A_{i+1}, B_{i+1})$ to $(A_n, B_n)$  such that $B_n = \{ f(a) \}$.
	\end{proof}
	
	
	
	
	
	\section{Bounded Anytime Exact BP}\label{AtBPMath}
	
	We now define bounding functions for the message functions in BP.
	
	\begin{definition}[Partition tree]
		For a carrier set $\mathbb{M}$, we inductively define \emph{partition tree on $\mathbb{M}$} as either of the following objects:
		\begin{itemize}
			
			\item a set $\mathbb{M} \setminus \Phi^{ext}$, where $\Phi^{ext} \subseteq \mathbb{M}$
			(that is to say, a set defined by exclusion).
			\item a tuple $(Pt_1,\dots,Pt_n)$ of partition trees on $\mathbb{M}$.
		\end{itemize}
	
	A \emph{complete} partition tree is a partition tree in which all the leaves are equal to $\mathbb{M} \setminus \mathbb{M} = \emptyset$.
	\end{definition}
	
	When $\mathbb{M}$ is a set of factors, we can annotate a partition tree
	with a set of external factors $\Phi^{ext}$ that do not occur in it, a set
	of external separator variables $D^{ext}$ present in more than one of its partition sub-trees \emph{and}
	also in the external factors $\Phi^{ext}$, and
	a set of internal separator variables $D$ present in more than one of its partition
	subtrees but not in the external factors $\Phi^{ext}$.

	A partition tree $Pt$ induces a set $completion(Pt)$ all the full partition trees consistent with $Pt$. A set of full partition trees can be used as a bound on a single, unknown full partition tree.
	
	We will use such bounds, represented by a partial partition tree, as the bounds $A_i$ on an unknown full partition tree passed to bounding functions for the message functions $m_{\phi \leftarrow V}$ and $m_{V \leftarrow \phi}$ .

	The bounding function obtains a partial partition tree $Pt_i$ representing an unknown full partition tree, chooses an $j$-th incoming message
	that has not converged yet, obtains a tighter bound $Pt^j$ and a tighter bound $B^j_{i+1}$ on
	the respective incoming message, and generates a tighter $Pt_{i+1}$ and bound $B_{i+1}$ on message as a result.
	
	For fixed $V$ and $\phi$, the message calculations $m^{Pt,D,D^{ext}}_{V \leftarrow \phi}$ and $m^{Pt,D,D^{ext}}_{\phi \leftarrow V}$ are functions on a complete $Pt$.
	The bounding functions can be defined as functions on partial $Pt$ as follows:
	
	\begin{theorem}[Anytime Belief Propagation]
		Given a set of factors $M$ and a cutset $Q$, $\bar{m}_{V}$ and $\bar{m}_{V \leftarrow \phi}$ are bounding functions with respect to $P_M(Q)$. (Note that $A_0 = 2^{\mathbb{M}} \times compl(Q)$ and $B_0 = \mathbb{P}(V)$.)
	\end{theorem}

	\begin{align*}
		& \bar{m}_{V \leftarrow \phi}^{Pt_i, D_i, D^{ext}_i} = ((Pt_{i+1},D_{i+1},D^{ext}_{i+1}),B_{i+1}) \\
		& \text{\qquad where } \\
		& \qquad \text{if $Pt_i$ is of the form $\mathbb{M} \setminus \Phi^{ext}_i$,}\\
		& \qquad \qquad \text{$\{S_{V \leftarrow \phi}^1,\dots,S_{V \leftarrow \phi}^n\}$ is $Var(\phi) \setminus (V \cup D^{ext}_i)$}\\
		& \qquad \qquad \text{$Pt^j_{i+1}$ is $\mathbb{M} \setminus (\Phi_i^0 \cup \{\phi\}), j = 1,\dots,n$}\\
		& \qquad \qquad \text{$Pt_{i+1}$ is $(Pt^1_{i+1},\dots,Pt^n_{i+1})$}\\
		& \qquad \qquad \text{$B_{i+1}$ is $\mathcal{P}(V)$}\\
		& \qquad \text{else} \\
		& \qquad j = \text{ index of incoming message not yet exhausted} \\
		& \qquad ((Pt_{i+1}^j, D_{i+1}^j, D_{i+1}^{ext,j}), B_{i+1}^j) = \bar{m}_{\phi \leftarrow S_i^j \cup D_i}^{Pt_i^j, D_i^j, D_i^{ext}} \\
		& \qquad ((Pt_{i+1}^k, D_{i+1}^k, D_{i+1}^{ext,k}), B_{i+1}^k) = ((Pt_{i}^k, D_{i}^k, D_{i}^{ext,k}), B_{i}^k),\\
        & \qquad \qquad \qquad \qquad \qquad \qquad \qquad \qquad \qquad \qquad \qquad k \neq j \\
		& \qquad D^{ext}_{i+1} = \bigl( \bigcup_j D^j_{i+1} \bigr) \cap Var(\Phi^{ext}_i)\\
		& \qquad D_{i+1} = \bigl( \bigcup_j D^j_{i+1} \bigr) \setminus D^{ext}_{i+1}\\
		& \qquad B_{i+i} = \sum_{D_{i+1}} \sum_{S_{V \leftarrow \phi} \setminus D_{i+1}} \phi \prod_{j=1}^n B_{i+1}^j \\
		\end{align*}
        \begin{align*}
		& \bar{m}_{\phi \leftarrow V}^{Pt_i, D_i, D^{ext}_i} = ((Pt_{i+1},D_{i+1},D^{ext}_{i+1}),B_{i+1})\\
		& \text{\qquad where } \\
		& \qquad \text{if $Pt_i$ is of the form $\mathbb{M} \setminus \Phi^{ext}_i$,}\\
		& \qquad \qquad \text{$\{\phi^1_{\phi \leftarrow V},\dots,\phi_{\phi \leftarrow V}^n\}$ is $neighbors_{Pt_i}(V) \setminus \{\phi\}$}\\
		& \qquad \qquad \text{$Pt^j_{i+1}$ is $\mathbb{M} \setminus (\Phi_i^{ext} \cup \bigcup_{k\neq j}\phi_{\phi \leftarrow V}^k), j = 1,\dots,n$}\\
		& \qquad \qquad \text{$Pt_{i+1}$ is $(Pt^1_{i+1},\dots,Pt^n_{i+1})$}\\
		& \qquad \qquad \text{$B_{i+1}$ is $\mathcal{P}(V)$}\\
		& \qquad \text{else} \\
		& \qquad j = \text{ index of incoming message not yet exhausted} \\
		& \qquad ((Pt_{i+1}^j, D_{i+1}^j, D_{i+1}^{ext,j}), B_{i+1}^j) = \bar{m}_{V \leftarrow \phi_{\phi \leftarrow V}^j}^{Pt_i^j, D_i^j, D_i^{ext}} \\
		& \qquad ((Pt_{i+1}^k, D_{i+1}^k, D_{i+1}^{ext,k}), B_{i+1}^k) = ((Pt_{i}^k, D_{i}^k, D_{i}^{ext,k}), B_{i}^k),\\
        & \qquad \qquad \qquad \qquad \qquad \qquad \qquad \qquad \qquad \qquad \qquad k \neq j \\
        & \qquad D^{ext}_{i+1} = \bigl( \bigcup_j D^j_{i+1} \bigr) \cap Var(\Phi^{ext}_i)\\
		& \qquad D_{i+1} = \bigl( \bigcup_j D^j_{i+1} \bigr) \setminus D^{ext}_{i+1}\\
		& \qquad \text{$Pt_{i+1}$ is $(Pt_1,\dots,Pt_n)$}\\
		& \qquad B_{i+i} = \sum_{D_{i+1}} \prod_{j=1}^n B_{i+1}^j
	\end{align*}	

\section{Computing with bounds}

In the algorithm described on section \ref{AtBPMath}, we make two operations that can be specially expensive:

$$B_{i+1} = \Bigl\{ \sum_N \phi(V, N) \prod_{N_j \in N} b^j : b^j \in B_{i}^j \Bigr\}$$
and
$$B_{i+i} = \Bigl\{ \prod_{j = 1..N} b^j : b^j \in B_{i}^j \Bigr\}$$

One can easily see that the computation of such expressions is not computationally possible in the general case: if we consider bounds with an infinite number of elements there would be an infinite number of computations to do.
However, this can be a relatively inexpensive computation in certain conditions, such as in the case where the $B$'s are convex sets with a finite number of extreme points. In this section we show that these conditions are fulfilled in our case and provide a way to preform this operations.

\subsection{The simplex bound}

We define a bound on a query Q as a set of non-normalized probabilities for that query. The simplest and most general bound we can assign to a query is the simplex relative to Q, noted as $\mathcal{P}_Q$. The simplex is defined as:

$$ \mathcal{P}_Q \equiv \{\phi(Q) : Q \rightarrow \mathbb{R}, \ \sum\limits_{q\in Q}\phi(q) = 1\} $$

One can easily see that a simplex is a convex set. Its extreme points are:

$$ext(\mP_Q) = \{\phi : \exists q\in Q \text{ s.t. } \phi(q) = 1 \text{ and }\forall p \neq q \  \phi(p) = 0 \}  $$
Which is a finite set.

\subsection{Operations with sets (definitions)}

This means that the probability distribution of $Q$ can be any function whose sum over the values of $Q$ is equal to one. That makes $\mathcal{P}_Q$ a trivial bound for $\Prob_M(Q)$.

We define the product of bounds as the set of products of each term from both sets. The product between a bound and a term is defined the same way.

$$B_1 \times B_2 \equiv \{\phi_1\times \phi_2 : \phi_1 \in B_1, \phi_2 \in B_2\}$$
$$B \times \phi \equiv \{\phi \times \phi' : \phi' \in B\}$$

We define the normalization operator $\mN$ in the following manner. Given a convex set functions $B$:

$$\mN(B) = \Bigr\{\frac{\phi}{\sum\limits_{domain(\phi)}\phi} : \phi \in B \Bigl\} $$

\subsection{Proving convexity}

Consider a function $\phi : D_{\phi} \rightarrow \mathbb{R}$, with $ D_{\phi}$ finite, and $B$ a convex set of functions $f:D_B\rightarrow \mathbb{R}$ with finite extreme points. Let $A \subset D_{\phi}$. 
. We prove that:
\begin{itemize}
\item $\mN(\sum_A \phi B)$ is convex 
\item $ext(\mN(\sum_A \phi B)) \subset \mN\sum_A \phi \ ext(B)$
\end{itemize}

We split the proof in two parts:

\subsubsection{1}
we want to prove that: 

\begin{itemize}
\item $\sum_A \phi B$ is convex :

Consider $\psi_1,...,\psi_k \in  \sum_A \phi B$. Then $\psi_i = \sum_A \phi \psi'_i$, $i = 1,...,k$, $\psi'_i \in B$. consider $\alpha_j>0$, $j= 1,..,k$ such that $\sum_j\alpha_j=1$.
Then:

$$\sum_{i= 1,...,k}\alpha_i\psi_i=\sum_{i= 1,...,k}\alpha_i\sum_A\phi\psi_i=\sum_A\phi\sum_{i= 1,...,k}\alpha_i\psi_i=$$

$$=\sum_A\phi(\psi'')\text{ , with } \psi''\in B$$

This implies that $\sum\limits_{i= 1,...,k}\alpha_i\psi_i \in \sum_A \phi B$, i.e, that $\sum_A \phi B$ is convex.

\item $ext(\sum_A \phi B) \subset \sum_A \phi \ ext(B)$

The proof follows by contradiction.

Consider $\psi \in ext(\sum_A\phi B)$ and suppose that $\psi \notin \sum_A\phi ext(B)$. Then $\psi = \sum_A\phi\psi'$, with $\psi' = \sum_i\alpha_i\psi'_i$, $\psi'_i \in B$ and $\alpha_i>0$ summing to 1. Then:

$$\psi = \sum_A\phi \sum_i\alpha_i\psi'_i = \sum_i\alpha_i(\sum_A\phi\psi'_i) = \sum_i\alpha_i\psi_i$$

with $\psi_i \in \sum_A\phi B$. Then $\psi$ is a convex combination of elements in $\sum_A\phi B$, which s a contradiction on the assumption that it is an extreme point.

\end{itemize}

\subsubsection{2}
we want to prove that: 

\begin{itemize}
\item $\mN(B)$ is convex 

We note $\sum\limits_{domain(\phi)}\phi = |\phi|$ for any function  $\phi$.

Let $\psi = \sum_i \alpha_i\psi_i$ be a convex combination of $\psi_i\in \mN(B)$. One can easily see that

$$\psi \in \mN(B) \iff |\psi| = 1 \text{ and } \psi\propto \psi' \text{ for some }\psi' \in B$$

We prove that $|\psi| = 1$ through the following equation:

$$|\psi| = |\sum_i\alpha_i\psi_i| = \sum_i\alpha_i|\psi_i| = \sum_i\alpha_i\times 1 = 1$$

Now we prove the second part of the equivalence stament:

We have $\forall i , \ \psi_i = \frac{\psi_i'}{|\psi_i'|} \text{ for some } \psi'_i \in B $. Then:

$$\psi = \sum_i \alpha_i\psi_i = \sum_i\frac{\alpha_i}{|\psi_i'|}\psi'_i = \sum_i\beta_i\psi'_i $$

Where $\beta_i = \frac{\alpha_i}{|\psi_i'|} $. Calling $ N = \sum_i\beta_i$, we have that $\sum_i \frac{\beta_i}{N}\psi_i \in B$ because it is a convex combination of elements in $S$. Therefore:

$$\psi = N \times (\sum_i \frac{\beta_i}{N}\psi_i) \ \ \text{ and } \ \ \sum_i \frac{\beta_i}{N}\psi_i \in B$$

Then:

$$\psi \propto \psi'  \text{ for some }  \psi'\in B$$

Which proves that $\mN(B)$ is convex.

\item $ext(\mN(B)) \subset \mN ext(B)$

Consider $\psi \in \mN(B)$. Then, $\exists \psi' \in B \ s.t. \ \psi = \frac{\psi'}{|\psi'|} $. According to Krein-Milman theorem, $\psi' = \sum\limits_{i=1,...,k}\alpha_i\psi'_i$ , $\alpha_i>0$, $\sum_i\alpha_i = 1$ and $\psi_i'\in  ext(B)$. Then:

$$ \psi = \frac{ \sum_i (|\psi'_i|\alpha_i)(\frac{\psi'_i}{|\psi'_i|})}{|\sum_i\alpha_i\psi'_i|} = \sum_ic_i\frac{\psi'_i}{|\psi'_i|}, \text{ with } c_i =  \frac{|\psi'_i|\alpha_i }{|\sum_i\alpha_i\psi'_i|}$$

We can easily see that $c_i > 0 $ and $\sum_ic_i = 1$ and that $\frac{\psi'_i}{|\psi'_i|} \in \mN(ext(B))$.
Also, $\mN(ext(B)) \subset \mN(B)$, since $\forall \psi' \in ext(B), \ \frac{\psi'}{|\psi'|} \in \mN(B)$. Then we have:
$$\mN(B) = C.Hull(\mN(ext(B))) $$
Then, because the extreme points of a convex se S are those that are not convex combinations of any others (except 1 $\times$ themselves), we have:
$$ext(\mN(B)) \subset \mN ext(B)$$

\end{itemize}

\subsection{Storing and computing the following bounds}

The operations with bounds can be resumed by applying the expression $\mN(\sum_A\phi B)$ for a certain set of variables $A$, a factor $\phi$ and a previous bound $B$. The last section showed that it suffices to store the extreme points of previous bounds and compute and store the new bound as a result of $\mN(\sum_A\phi \psi)$ for each $\psi \in ext(B)$.

}

\end{document}